\documentclass[10pt,twocolumn,letterpaper]{article}

\usepackage[pagenumbers]{cvpr} 

\usepackage{graphicx}
\usepackage{amsmath}
\usepackage{amssymb}
\usepackage{booktabs}

\usepackage{bm}
\usepackage{multirow}

\usepackage{algorithm}
\usepackage{algorithmic}
\usepackage{amsmath}
\usepackage{amsthm}

\usepackage{comment}

\usepackage{array}

\newtheorem{theorem}{Theorem}

\usepackage[pagebackref,breaklinks,colorlinks]{hyperref}

\usepackage[capitalize]{cleveref}
\crefname{section}{Sec.}{Secs.}
\Crefname{section}{Section}{Sections}
\Crefname{table}{Table}{Tables}
\crefname{table}{Tab.}{Tabs.}

\def\cvprPaperID{4430} 
\def\confName{CVPR}
\def\confYear{2022}

\newcommand{\norm}[1]{\left\lVert#1\right\rVert}
\newcommand{\abs}[1]{\left|#1\right|}

\newcommand{\qileft}{[\kern-0.15em[}
\newcommand{\qiLeft}{\left[\kern-0.4em\left[}
\newcommand{\qiright}{]\kern-0.15em]}
\newcommand{\qiRight}{\right]\kern-0.4em\right]}

\newcommand{\diag}{{\mbox{diag}}}

\newcommand{\m}{{\bm{m}}}
\newcommand{\s}{{\bm{s}}}

\newcommand{\B}{\mathcal{B}}

\newcommand{\D}{\mathcal{D}}
\newcommand{\E}{\mathcal{E}}
\newcommand{\F}{\mathcal{F}}
\newcommand{\G}{\mathcal{G}}

\renewcommand{\O}{\mathcal{O}}

\newcommand{\R}{\mathbb{R}}

\newcommand{\T}{\mathcal{T}}

\newcommand{\MSA}{{\rm MSA}}
\newcommand{\MLP}{{\rm MLP}}

\newcommand{\HZ}{\widehat{Z}}
\newcommand{\HB}{\widehat{\B}}
\newcommand{\HF}{\widehat{\F}}
\newcommand{\st}{{\rm s.t.}}

\begin{document}

\title{Patch Slimming for Efficient Vision Transformers}

\author{
	Yehui Tang\textsuperscript{\rm 1,2},
	Kai Han\textsuperscript{\rm 2},
	Yunhe Wang\textsuperscript{\rm 2}\thanks{Corresponding author.},
	Chang Xu\textsuperscript{\rm 3},\\
	Jianyuan Guo\textsuperscript{\rm 2,3},
	Chao Xu\textsuperscript{\rm 1}, 
	Dacheng Tao\textsuperscript{\rm 4},
	\\
	\textsuperscript{\rm 1}School of Artificial Intelligence, Peking University. \textsuperscript{\rm 2}Huawei Noah’s Ark Lab.\\
	\textsuperscript{\rm 3}School of Computer Science, University of Sydney. \textsuperscript{\rm 4}JD Explore Academy, China. \\
	yhtang@pku.edu.cn, \{kai.han, yunhe.wang\}@huawei.com,  dacheng.tao@gmail.com.
}

\maketitle

\begin{abstract}
	This paper studies the efficiency problem for visual transformers by excavating redundant calculation in given networks. The recent transformer architecture has demonstrated its effectiveness for achieving excellent performance on a series of computer vision tasks. However, similar to that of convolutional neural networks, the huge computational cost of vision transformers is still a severe issue. Considering that the attention mechanism aggregates different patches layer-by-layer,  we present a novel patch slimming approach that discards useless patches in a top-down paradigm. We first identify the effective patches in the last layer and then use them to guide the patch selection process of previous layers. For each layer, the impact of a patch on the final output feature is approximated and patches with less impact will be  removed.  Experimental results on benchmark datasets demonstrate that the proposed method can significantly reduce the computational costs of vision transformers without affecting their performances. For example, over 45\% FLOPs of the  ViT-Ti model can be reduced with only 0.2\% top-1 accuracy drop on the ImageNet dataset.
	
\end{abstract}

\section{Introduction} 

\label{sec-intro}

Recently, transformer models have been introduced into the field of computer vision and achieved high performance in many tasks such as object recognition~\cite{dosovitskiy2020image}, image process~\cite{chen2020pre}, and video analysis~\cite{kim2018spatio}. Compared with the convolutional neural networks~(CNNs), the transformer architecture introduces less inductive biases  and hence has larger potential to absorb more training data and generalize well on  more diverse tasks~\cite{dosovitskiy2020image,touvron2020training,liu2021post,yuan2021tokens,tang2021augmented,han2022pyramidtnt}. However, similar to CNNs, vision transformers also suffer high computational cost, which blocks their deployment on resource-limited devices such as mobile phones and various IoT devices.  To apply a deep neural network in such real scenarios, massive model compression algorithms have been proposed to reduce the required computational cost~\cite{ye2018rethinking,li2016pruning}. For example, quantization algorithms approximate weights and intermediate features maps in neural networks with low-bit data~\cite{courbariaux2016binarized,rastegari2016xnor}. Knowledge distillation improves the performance of a compact network by transferring knowledge from giant models~\cite{hinton2015distilling,lan2018knowledge}.

In addition, network pruning is widely explored and used to reduce the neural architecture by directly removing useless components in the pre-defined network~\cite{han2015learning, han2015deep, liu2017learning,liu2018frequency}.  Structured pruning  discards whole contiguous components of a pre-trained model, which has attracted much attention in recent years, as it can realize acceleration without specific hardware design. In CNNs, removing a whole filter for improving the network efficiency is a representative paradigm, named channel pruning (or filter pruning)~\cite{he2017channel,liu2017learning}. For example, Liu~\etal~\cite{liu2017learning} introduce scaling factors to control the information flow in the neural network and filters with small factors will be removed.  Although the aforementioned network compression methods have made tremendous efforts for deploying compact convolutional neural networks, there are only few works discussing how to  accelerate vision transformers.

Different from the paradigm in conventional CNNs, the vision transformer splits the input image into multiple  patches and calculates the features of all these patches in parallel. The attention mechanism will further aggregate all patch embeddings into visual features as the output. Elements in the attention map reflect the relationship or similarity between any two patches, and the largest attention value for constructing the feature of an arbitrary patch is usually calculated from itself. Thus, we have to preserve this information flow in the pruned vision transformers for retaining the model performance, which cannot be guaranteed in the conventional CNN channel pruning methods. Moreover, not all the manually divided patches are informative enough and deserve to be preserved in all layers, \eg, some patches are redundant with others. Hence we consider developing a patch slimming approach that can effectively identify and remove redundant patches.

In this paper, we present a novel patch slimming algorithm for  accelerating the vision transformers. In contrast to existing works focusing on the redundancy in the network channel dimension, we aim to  explore the computational redundancy in the patches of a vision transformer~(as shown in Figure~\ref{fig-patch}. The proposed method removes redundant patches from the given transformer architecture in a top-down framework, in order to ensure the retained high-level features of discriminative patches can be well calculated. Specifically, the patch pruning will execute from the last layer to the first layer, wherein the useless patches are identified by calculating their importance scores to the final classification feature (\ie, class token). To guarantee the information flow, a patch will be preserved if the patches in the same spatial location are retained by deeper layers. For other patches, the importance scores determine whether they are preserved, and patches with lower scores will be discarded. The whole pruning scheme for vision transformers is conducted under a careful control of the network error, so that the pruned transformer network can maintain the original performance with significantly lower computational cost. Extensive experiments validate the effectiveness of the proposed method for deploying efficient vision transformers. For example, our method can reduce more than 45\% FLOPs of the ViT-Ti model with only 0.2\% top-1 accuracy loss on the ImageNet dataset.

\begin{figure}
	\centering
	\small	
	\includegraphics[width=0.7\columnwidth]{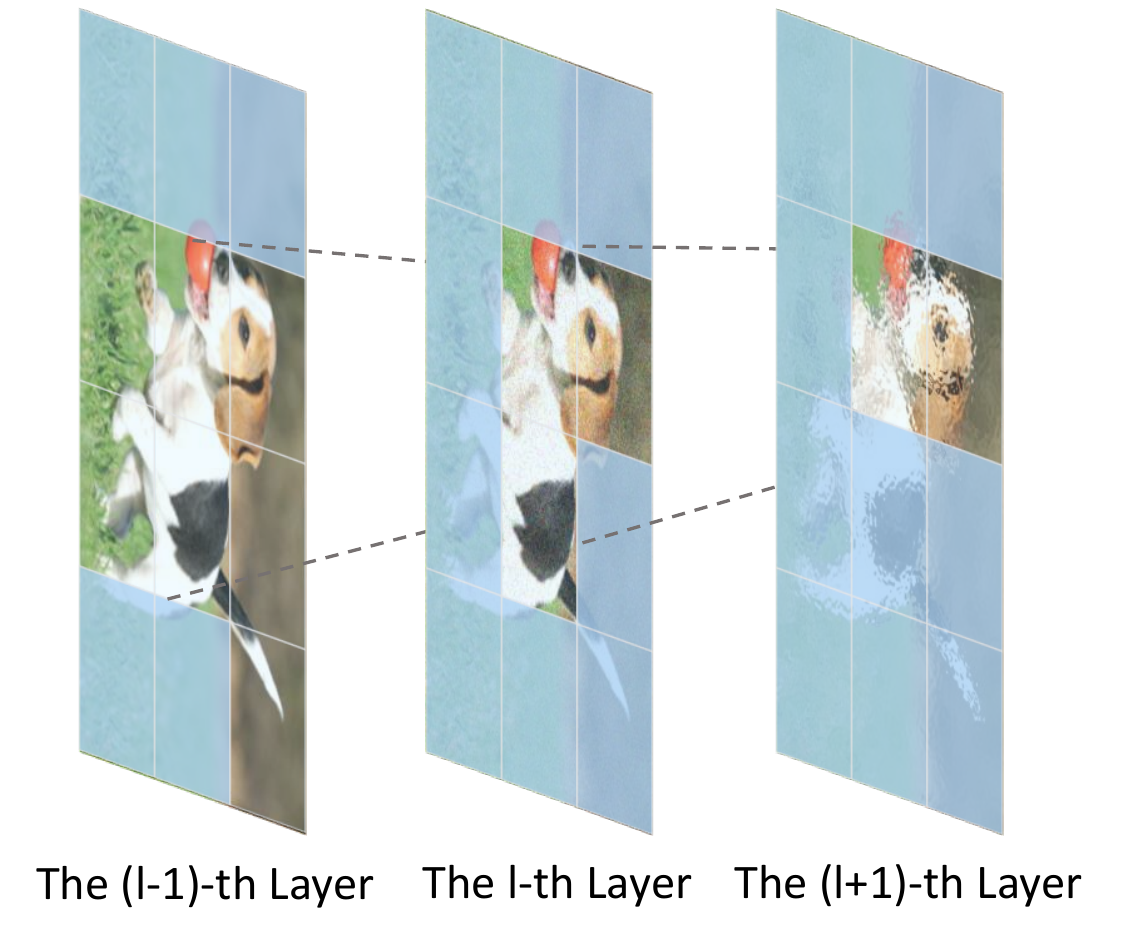}
	\caption{The diagram of  patch slimming for vision transformers.}
	\label{fig-patch}
\end{figure}
\section{Related work}

\textbf{Structure pruning for CNNs. }
Channel pruning discards the entire convolution kernels~\cite{lebedev2016fast} to accelerate the inference process and reduce the required memory cost~\cite{lebedev2016fast,DBLP:conf/iclr/SuYH00Z021,liu2017learning,he2019filter,su2021bcnet}. To identify the redundant filters, massive methods have been proposed.  Wen~\etal~\cite{wen2016learning} add a group-sparse regularization on the filters and remove filters with small norm. Beyond imposing sparsity regularization on the filters directly,  Liu~\etal~\cite{liu2017learning} introduce extra scaling factors to each channel and these scaling factors are trained to be sparse. Filters with small scaling factors has less impact on the network output and will be removed for accelerating inference. He~\etal~\cite{he2019filter} rethink the criterion that filters with small norm values are less important and propose to discard the filters having larger similarity to others. To  maximally excavate redundancy, Tang~\cite{tang2020scop} set up a scientific control to  alleviate the distribution of irrelevant factors and remove  filters with little relation to the given task. In the conventional channel pruning for CNNs, channels in different layers have no one-to-one relationship, and then the choice of effective channels in a layer has little impact on that in other channels.

\textbf{Structure pruning for transformers. }
In the transformer model for NLP tasks, a series of works focus on reducing the heads in the multi-head attention (MSA) module. For example, Michel~\etal~\cite{NEURIPS2019_2c601ad9} observes that removing a large percentages of heads in the  pre-trained BERT~\cite{devlin2018bert} models has limited impact on its performance. Voita~\etal~\cite{voita2019analyzing} analyze the role of each head in the transformer and  evaluate their contribution to the model performance. Those heads with less contributions will be reduced. Besides the MSA module, the neurons in the multilayer perceptron (MLP) module are also pruned in \cite{brown2020language}. Designed for vision transformers, VTP~\cite{zhu2021visual} reduces the number of embedding dimensions by introducing  control coefficients and removes neurons with small coefficients. Different from them, the proposed patch slimming explores the redundancy from a new perspective by considering the information integration of different patches in a vision transformer. Actually, reducing patches can be also combined with pruning in other dimensions to realize higher acceleration.

\section{Patch Slimming for Vision Transformer}
In this section, we introduce the scheme of pruning patches in vision transformers. We first review the vision transformer briefly and then introduce the formulation of patch slimming.

In vision transformer, the input image is split into $N$ patches and then fed into transformer model for representation learning. For an $L$-layer vision transformer model, the multihead self-attention~(MSA) modules and  multi-layer perceptron (MLP) modules are its main components occupying most of the computational cost. Denoting $Z_{l-1}, Z'_l \in \R^{N\times d}$  as the input and the intermediate features of the $l$-th layer, the MSA and MLP modules can be formulated as:
\begin{equation}
	\small
	\label{eq-blk}
	\begin{aligned}
	 &\MSA(Z_{l-1})\\
	 &\quad = {\rm Concat }\left[ {\rm softmax }\left( \frac{Q_l^h {K_l^h}^{\top}}{\sqrt{d}} \right)V_l^h \right]_{h=1}^H W_l^o, \\
	&\MLP (Z'_l)=  \phi (Z'_lW_l^a)W_l^b,
	\end{aligned}
\end{equation}
where $d$ is embedding dimension, $H$ is the number of heads, $Q_l^h=Z_{l-1}W^{hq}_l$, $K_l^h=Z_{l-1}W^{hk}_l$, and $V_l^h=Z_{l-1}W^{hv}_l$ are the query, key and value of the $h$-th head in the $l$-th layer, respectively. $W_l^a$, $W_l^b$ are the weights for linear transformation and $\phi(\cdot)$ is the non-linear activation function~(\eg, GeLU). Most of recent vision transformer models are constructed by stacking MSA and MLP modules alternately and  a block $\B_l (\cdot)$ is defined as $\B_l (Z_{l-1})=\MLP(\MSA(Z_{l-1})+Z_{l-1})+Z'_l$.

As discussed above, there is considerable redundant information existing in the patch level of vision transformers. To further verify this phenomenon, we calculate the average cosine similarity between patches within a layer, and show how similarity vary \wrt layers in Figure~\ref{fig-sim}. The similarity between patches increase rapidly as layers increase, and the average similarity even exceed 0.8 in deeper layers. The high similarity implies that patches are redundant especially in the deeper layers and removing them will not obviously affect the feature calculation.

Patch slimming aims to recognize and discard redundant patches for accelerating the inference process~(as shown in Figure~\ref{fig-patch}). Here we use a binary vector  $\m_l \in \{0,1\}^N$ to indicate whether a patch is preserved or not, the pruned MSA and MLP modules can be formulated as follows: 
\begin{equation}
	\small
	\label{eq-mblk}
	\begin{aligned} 
		&\widehat \MSA_l(\HZ_{l-1})
		\\&\ = {\rm Concat }\left[ \diag(\m_l) {\rm softmax }\left( \frac{{Q}_l^h {K_l^h}^{\top}}{\sqrt{d}} \right) {V}_l^h \right]_{h=1}^H W_l^o,\\ 
		&\widehat \MLP_l (\HZ'_l)= \diag(\m_l) \phi (\HZ'_lW_l^a)W_l^b,
	\end{aligned}
\end{equation} 
where $\diag(\m_l)$ is a diagonal matrix whose  diagonal line is composed of elements in $\m_l$. Specifically, $\m_{l,i}=0$ indicates that the $i$-th patch in the $l$-th layer is pruned. $\HZ_{l-1}$, $\HZ'_l$ are the input and the intermediate features of the $l$-th layer in a pruned vision transformer.
Then the pruned block is defined as $\HB_l(\HZ_{l-1})=\widehat \MLP_l(\widehat \MSA_l(\HZ_{l-1})+ \HZ_{l-1})+\HZ'_l$. 

In practical implementation, only the effective patches of input feature $\HZ_{l-1}$ are selected to calculate queries, and then all the subsequent operations are only implemented on these effective patches. 
Thus, the computation of the pruned patches can be avoided \footnote{According to $\m_l$, only effective patches from the shortcut branch are added to the output of pruned MSA, while the output of pruned MLP is padded with zeros before added to the shortcut.}.  

\textbf{Computation Efficiency.} Compared with the original block $\B_l(\cdot)$, the pruned $\HB_l(\cdot)$ can save a large amount of computational cost. Given a block $\B(\cdot)$ with $N$ patches and $d$-dimension embedding, the computational costs of MLP~(2-layers with hidden dimension $d'$) and MSA are $(2Ndd')$ and $(2N^2d+4Nd^2)$, respectively.  After pruning $\eta\%$ patches, all the computational components in MLP are pruned, and then $\eta\%$~FLOPs in the MLP module are reduced. For the MSA module, the cost of calculating query, attention map and output projection can be reduced, and then $\eta\%(2N^2d+ 2Nd^2)$ FLOPs is reduced.

\section{Excavating Redundancy via Inverse Pruning} 

In this section, we present the top-down framework to prune patches in the vision transformer, and provide an effective importance score estimation of each patch.

\begin{figure}[htp] 
	\centering
	\small
	\includegraphics[width=1.0\columnwidth]{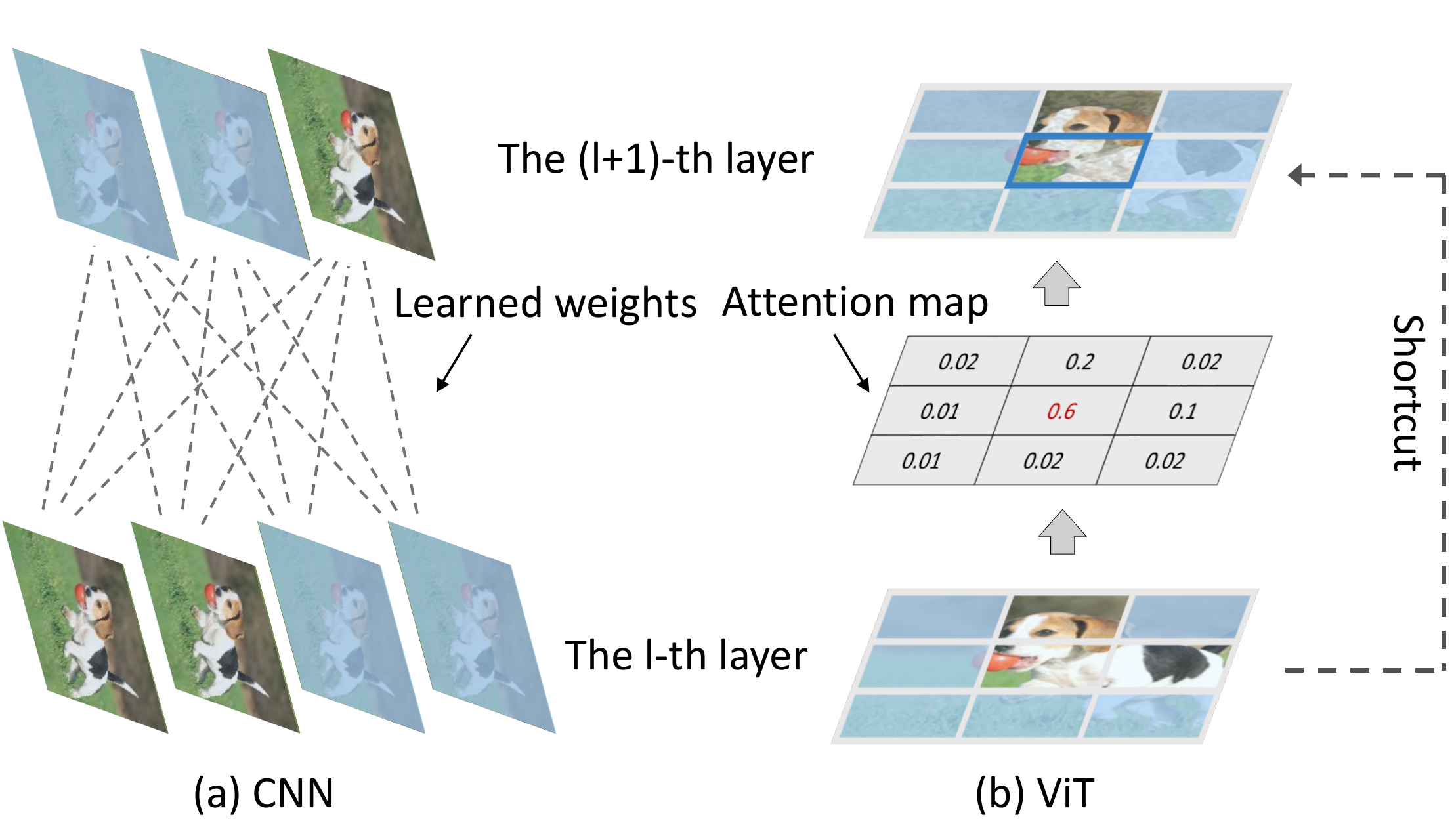}
	\caption{The comparison between channel pruning in CNNs and patch pruning in vision transformers.}
	\label{fig-cvt}
\end{figure}

\begin{figure}[htp] 
	\centering
	\small

	\includegraphics[width=0.6\columnwidth]{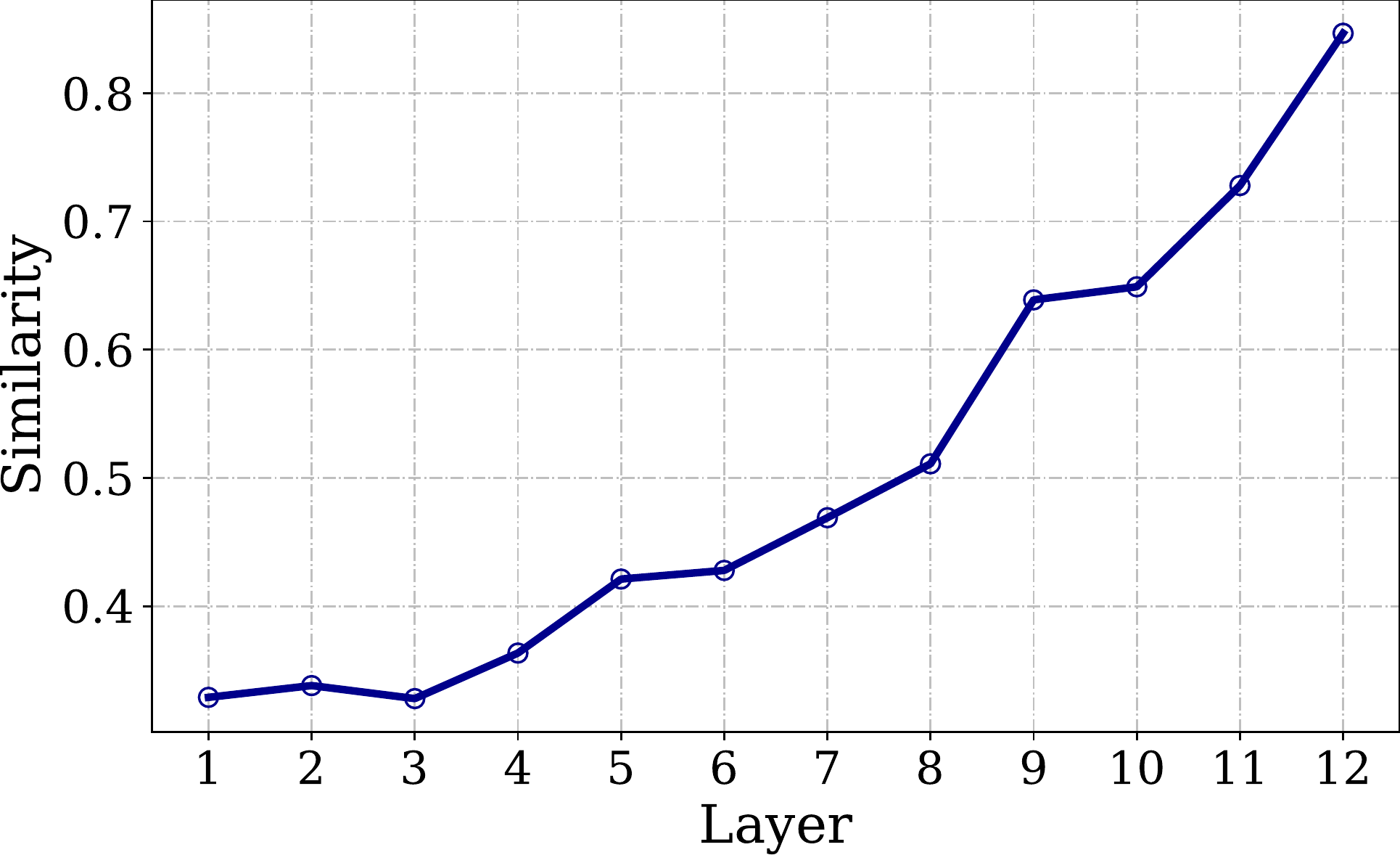}
	\caption{The average similarity of different patches varies \wrt network depth in the ViT-Base model.}
	\label{fig-sim}
\vspace{-2mm}
\end{figure}

\subsection{Top-Down Pruning}
\label{sec-proc}

For patch slimming in vision transformer, we adopt a top-down manner to prune patches layer-by-layer. It is a natural choice with two reasons as described in the following.

For a CNN model, pruning channels in different layers independently can achieves high performance ~\cite{liu2017learning,tang2020scop}. However, this paradigm cannot work well in vision transformers. The main reason is that patches in different layers of a vision transformer are one-to-one corresponding. 
Figure~\ref{fig-cvt}  compares pruning channels in CNNs and  pruning patches  in vision transformers.
As own in Figure~\ref{fig-cvt}(a), channels in adjacent layers of a CNN model are fully connected by learnable weights, and each channel contains information from the entire image. However, in the vision transformer~(Figure~\ref{fig-cvt}(b)), different patches communicate with others by an attention map, which reflects the similarity between different patches. If patch $i$ and patch $j$ are more similar, the corresponding value $A_{lh}^{ij}$ tends to have a larger value. The diagonal elements $A_{lh}^{ii}$ usually plays a dominant role, that is, a patch pays highest attention to the input at the position of itself. Besides, the shortcut connection directly copies the feature in the $l$-layer to the corresponding patches in the next layer. This one-to-one correspondence inspires us to preserve some important patches in the same spatial locations of different layers, which can guarantee the information propagation across layers.

Another characteristic of  vision transformer is that deeper layers tend to have more redundant patches. The attention mechanism in the MSA module aggregates different patches layer-by-layer, and a large number of similar patches are produced in the process (as shown in Figure~\ref{fig-sim}).  It implies that more redundant patches can be safely removed in deeper layers, and fewer in shallower layers. 

Based on the above analysis, we start the pruning procedure from the output layer, and then prune previous layers by transmitting the selected effective patches from top to down. Specially, all the patches preserved in the $(l+1)$-th layer will be also preserved in the $l$-th layer. 
Thus, this top-down pruning procedure can guarantee that shallow layers maintain more patches than the deep layers, which is consistent with the redundancy characteristic of vision transformer. 

\subsection{Impact Estimation}
\label{sec-impact}
With the patch pruning scheme described in the above section, all that's left is to recognize redundant patches in a vision transformer, \ie, find the optimal mask $\m_l$ in each layer. Our goal is to prune patches as many as possible to realize maximal acceleration, while maintaining the representation ability of the output feature. Actually, only a part of patch embeddings in the last layer are used to predict the labels of input images for a specific task. For example, in the image classification task, only a patch related to classification (\ie, class token) is sent to the classifier for predicting labels. Other patches in the output layer can be removed safely without affecting network output. Supposing the first patch is the class token, we can get the mask in the last layer, \ie, $\m_{L,1}=1$, and $\m_{L,i}=0, \forall~i=2,3,\cdots,N$. Then for the other layers, the optimization object is formulated as follows:
\begin{equation}
	\begin{aligned}
		\label{eq-obj} 
		&\min_{\m_1, \m_2, \cdots, \m_{L-1}} \sum_{l=1}^{L-1}\|\m_l\|_0,\\ 
		&\quad \quad\st \ \m_l\in\{0,1\}^N , \\ 
		&\quad \quad\E_L=\norm{\diag(\m_L)\left(\HZ_L-Z_L\right)}_F^2\le \epsilon,
	\end{aligned}
\end{equation}
where  $\|\cdot\|_0$ is the $\ell_0$-norm of a vector, \ie, the number of non-zero elements. $\|\cdot \|_F$ is the Frobenius norm of a matrix. $\epsilon$ is the tolerable error. $\HZ_L$ and $Z_L$ are output features of the pruned and unpruned transformers.   Eq.~\ref{eq-obj} is hard to optimize directly, as it involve $\ell_0$ optimization under constraint, which is non-convex, NP hard and requires combinatorial search~\cite{liu2017learning}.  To solve Eq.~\ref{eq-obj}, we firstly define a signification score by approximating the impact of a patch on the reconstruction error $\E_L$  and then develop a pruning procedure.  

The attention mechanism aggregates  information from different patches to one patch, which is the main cause to produce redundant patch features. To focus on the attention layer for excavating redundant patches, we reformulate the definition of a block $\B(\cdot)$ in a simple formulation. Denoting $P_l^h= {\rm softmax }\left( {Q^{h}_l{K^{h}_l}^{\top}}/{\sqrt{d}}\right)$, the 
MSA module in Eq.~\ref{eq-blk} can be formulated as:
\begin{equation}
	\small
	\begin{aligned}
	\MSA(Z_l)& = {\rm Concat }\left [P_l^h V_l^h\right]_{h=1}^H W^{o}_l \\
	&= \sum_{h=1}^H P_l^h V_l^h W^{ho}_l
	 = \sum_{h=1}^H P_l^h Z_{l-1} W_l^{hv} W^{ho}_l,
	\end{aligned}
\end{equation} 
where $W^{o}_l=[W^{1o}_l;W^{2o}_l; \cdots; W^{Ho}_l]$, and $W^{ho}_l\in\mathbb{R}^{\frac{d}{H}\times d}$.
Then the original block $\B_l(Z_{l-1})$ and pruned block $\HB_l(Z_{l-1})$ can be represented as:
\begin{equation}
	\small
	\label{eq-sblk}
	\begin{aligned}
	&\B_l(Z_{l-1})=\O(\sum_{h=1}^HP_l^hZ_{l-1},\{W_l\}), \\ &\HB_l(Z_{l-1},\m_l)=\O(\sum_{h=1}^H\diag(\m_l)P_l^hZ_{l-1},\{W_l\})
	\end{aligned}
\end{equation}
where $\O(\cdot,W_l)$  is composed of multiple linear projection matrices $\{W_l\}$ in the MSA and MLP module, as well as non-linear activation functions (\eg, GeLU). 
 
Based on the simplified formulation of a block~(Eq.~\ref{eq-sblk}), we here explore how a patch in the $t$-th layer affects the error $\E_L$~(Eq~\ref{eq-obj}) of effective patches in the last layer. We reverse the transformer and prune it from the last to the first layer sequentially. Thus when it comes to the $t$-th layer, all the deeper layers have been pruned. To approximate the significance of each token, we have the following theorem.

\begin{theorem}
	The impact of the $t$-th layer's patch on the final error $\E_L$ can be reflected by a significance metric $\s_t\in\R^N$. For the $i$-th patch in the $t$-th layer, we have 
\begin{equation}
	\label{eq-imp}
	\s_{t,i}=\sum_{h\in [H]^{L\sim t+1}}  \norm{A_t^h[:,i] U_t^h[i,:]}_F^2,
\end{equation} 
where $A_t^h=\prod_{l=t+1}^L\diag(\m_l)P_l^h $ and $U_t^h=P_t^h \abs{Z_{t-1}}$.  $A_t^h[:,i]$ denotes the $i$-th column of $A_t^h$  and $U^h_t[i,:]$ is $i$-th row of $U_t^h$. $[H]^{L\sim t+1}$ denotes all the attention heads in the $(t+1)$-th to $L$-th layer.
\end{theorem} 

\begin{proof}
	We use $\HF_{l\sim t}(Z_{t-1},\{\m_t\}_t^L)~(l>t)$ to denote  feature of the $l$-th layer in a vision transformer, whose layers behind $t$-th layer have been pruned, while the previous layers has not pruned yet, \ie, $\HF_{L\sim t}(Z_{t-1},\{\m_t\}_t^L) = \hat \B_L \circ \hat \B_{L-1} \circ \cdots \circ \hat \B_{t}(Z_{t-1})$.
	When pruning the patch in the $t$-th layer, we compare effective patches of the last layer from two transformers to decide whether the $t$-th layer has been pruned. Then the error $\E_L$ is calculated as:
	\begin{equation}
		\small
		\begin{aligned}
		\E_L &=||\diag(\m_L)[\HF_{L\sim (t+1)}(\HB_t(Z_{t-1})) \\
		&-  \HF_{L\sim (t+1)}(\B_t(Z_{t-1}))]||_F^2.
		\end{aligned}
	\end{equation} 	
	The error $\E_L$ in the last layer can be represented by the patches in the $(L-1)$-th layer, \ie,
	\begin{equation}
		\small
		\label{eq-el1}
		\begin{aligned}
			\E_L& =  ||\diag(\m_L) P_{L}^h   [\O(\sum_{h=1}^H    \HF_{(L-1)\sim (t+1)}(\HB_t(Z_{t-1})))\\
			 &-\O( \sum_{h=1}^H  \HF_{(L-1)\sim (t+1)}(\B_t(Z_{t-1})) )  ] ||_F^2 \\
			&\le C_L  || \sum_{h=1}^H \diag(\m_L) P_{L}^h |\HF_{(L-1)\sim (t+1)}(\HB_t(Z_{t-1}))\\ 
			&- \HF_{(L-1)\sim (t+1)}(\B_t(Z_{t-1}))|||_F^2, 
		\end{aligned}
	\end{equation}
	where $|\cdot|$ is the element-wisely absolute value. The inequality above comes the Lipschitz continuity~\cite{funahashi1993approximation, dupuis1991lipschitz} of function $\O(\cdot)$ and $C_L$ is the Lipschitz constant. Recalling that $\O(\cdot)$ is compose of multiple linear projections and non-linear activation function, the condition of Lipschitz continuity is satisfied~\cite{funahashi1993approximation}. $E_L$ can be further transmitted to previous layers, and for the $t$-th layer we have 
	\begin{align}
		\small
		\E_L & \le  \prod_{l=t+1}^L C_l  || \sum_{h\in [H]^{L\sim t}}  \prod_{l=t+1}^L\diag(\m_l)P_l^h \\
		& \quad |\HB_t(Z_{t-1})-\B_t(Z_{t-1})| ||_F^2 \\
		& \le  \prod_{l=t}^L C_l  || \sum_{h\in [H]^{L\sim t+1}}  \prod_{l=t+1}^L\diag(\m_l) \\
		& \quad P_l^h \left (I_N-\diag(\m_l)\right)P_t^h \abs{Z_{t-1}} ||_F^2 \\
		&=C'_t  ||\sum_{h\in [H]^{L\sim t+1}}A_t^h (I_N -\diag(\m_t)) U_t^h ||_F^2, \label{eq-ela}	
	\end{align}
	where $A_t^h=\prod_{l=t+1}^L\diag(\m_l)P_l^h \in \R^{N\times N}$,  $U_t^h=P_t^h \abs{Z_{t-1}} \in \R^{N\times d}$, and $C'_t=\prod_{l=t}^L C_l$. $[H]^{l\sim t}$ denotes all the attention heads in the $t$-th to $l$-th layer. To investigate how each patch in the $t$-th layer affect the final error $E_L$, we expand Eq.~\ref{eq-ela} \wrt each element in the indicator $\m_l$. Denoting $\m_{l,i}$ as the $i$-th element in $\m_l$, $A_t^h[:,i]$ is the $i$-th column of $A_t^h$  and $U^t_t[i,:]$ is $i$-th row of $U_t^h$, Eq.~\ref{eq-ela} can be written as:
	\begin{align}
		\small
		\E_L & \le C'_t  ||\sum_{h\in [H]^{L\sim t+1}} \sum_{i=1}^N A_t^h[:,i] (1-\m_{t,i}) U_t^h[i,:] ||_F^2 \\
		& \le C'_t \sum_{i=1}^N(1-\m_{t,i})\sum_{h\in [H]^{L\sim t+1}}  \norm{A_t^h[:,i] U_t^h[i,:]}_F^2 \label{eq-ele}.
	\end{align}
	Then we get the importance of each patch, \ie, $	\s_{t,i}=\sum_{h\in [H]^{L\sim t+1}}  \norm{A_t^h[:,i] U_t^h[i,:]}_F^2$.
\end{proof}
\begin{table*}[t] 
	\centering
	\small 
	\caption{Comparison of the pruned vision transformers with different methods on ImageNet. `FLOPs~$\downarrow$' denotes the reduction ratio of FLOPs.} 
	\vspace{-2mm}
	\begin{tabular}{c|c|c|c|c|c|c|c}
		\toprule[1.5pt]
		
		\multirow{2}{*}{Model} &   \multirow{2}{*}{Method}  & {Top-1 } &Top-5   &  FLOPs  & FLOPs &Throughput& Throughput \\ 
		&&Acc. (\%)&Acc. (\%)&(G)&$\downarrow$ (\%)&(image / s)& $\uparrow$ (\%)\\
		\hline
		\multirow{6}{*}{ViT~(DeiT)-Ti}  
		& Baseline& 72.2&91.1&  1.3 & 0& 2536&0\\
		&SCOP~\cite{tang2020scop} &68.9 (-3.3)&89.0 (-2.1)&0.8&38.4&3372&33.0 \\
		&PoWER~\cite{goyal2020power}&69.4 (-2.8)&89.2 (-1.9)&0.8&38.4&3304&30.3\\
		&HVT~\cite{pan2021scalable} & 69.7 (-2.5) & 89.4 (-1.7)  & 0.7 & 46.2&3524&38.9\\
		& PS-ViT (Ours) &\textbf{72.0} (-0.2)&\textbf{91.0} (-0.1)&0.7&46.2&3576&41.0\\
		& DPS-ViT (Ours) &\textbf{72.1} (-0.1)&\textbf{91.1} (-0.0) &0.6&53.8&3639&43.5 \\ 
		
		\hline 
		\multirow{6}{*}{ViT~(DeiT)-S}  
		& Baseline& 79.8&95.0&  4.6 & 0&940&0\\
		&SCOP~\cite{tang2020scop} &77.5 (-2.3)&93.5 (-1.5)&2.6&43.6&1310&39.4 \\
		&PoWER~\cite{goyal2020power} &78.3 (-1.5)&94.0 (-1.0)& 2.7 & 41.3&1295&37.8\\
		&HVT~\cite{pan2021scalable} & 78.0 (-1.8) & 93.8 (-1.2)  &  2.4 & 47.8&1335&42.1 \\
		&PS-ViT (Ours) & \textbf{79.4} (-0.4) & \textbf{94.7} (-0.3)  &  2.6 & 43.6&1321&40.5 \\	
		&DPS-ViT (Ours) & \textbf{79.5} (-0.3) & \textbf{94.8} (-0.2)  &  2.4 & 47.8&1342&42.8 \\	
		\hline 
		\multirow{6}{*}{ViT~(DeiT)-B}  
		& Baseline&81.8 &95.6 &  17.6&0&292&0\\
		&SCOP~\cite{tang2020scop} &79.7 (-2.1)&94.5 (-1.1)&10.2& 42.0&403&38.1\\
		&PoWER~\cite{goyal2020power}&80.1 (-1.7)&94.6 (-1.0)&10.4&39.2&397&35.8\\	
		&VTP~\cite{zhu2021visual}&80.7  (-1.1)&95.0 (-0.6)&10.0& 43.2&412&41.0\\ 
		&PS-ViT (Ours)  & \textbf{81.5} (-0.3) & \textbf{95.4} (-0.2)& 9.8 &44.3&414&41.8\\ 
		&DPS-ViT (Ours)  & \textbf{81.6} (-0.2) & \textbf{95.4} (-0.2)& 9.4 &46.6&413&41.3\\ 	
		\hline 
		\multirow{4}{*}{T2T-ViT-14}  
		& Baseline&81.5 &95.4& 5.2 &0&764&0\\ 
		&PoWER~\cite{goyal2020power} &79.9 (-1.6)&94.4 (-1.0)&3.5&32.7&991&29.7\\
		&PS-T2T (Ours) &\textbf{81.1} (-0.4)&\textbf{95.2} (-0.2)&3.1&40.4&1055&38.1\\ 
		&DPS-T2T (Ours) &\textbf{81.3} (-0.2)&\textbf{95.3} (-0.1)&3.1&45.4&1078&41.1\\  
		
		\bottomrule[1.5pt]	
	\end{tabular}
	\vspace{-4mm}
	\label{tab-img}
\end{table*}
For the $i$-th patch in the $t$-th layer,  $\s_{t,i}$ reflects its impact on the effective output of the final layer.
A larger  $\s_{t,i}$ implies the corresponding patch has larger impact to the final error, which can reflect the importance of a patch to the model performance. The calculation of $\s_{t,i}$ involves all the attention maps in behind layers and the input feature of the current layer. Before pruning the current layer, we randomly sample a subset of training dataset to calculate the significance scores $\s_t$ and the average $\s_t$ over these data is adopted. The obtained $\s_t$ can be viewed as the real-number score for binary $\m_t$.

\begin{algorithm}[h]
	\caption{Patch Slimming for Vision Transformers.}
	\label{alg}
	\begin{algorithmic}[1] 
		\REQUIRE{Training dataset $\D$, vision transformer $\T$ with L layers, patch masks $\{\m_l\}_{l=1}^L$, tolerant value $\epsilon$, preserved patch's number $r$ and search granularity $r'$. }
		\STATE Initialize $\m_{L,0}$ as 1 and other elements as 0.
		\FOR{ $l=L-1, \cdots, 1$ }
		\STATE Randomly sample a subset of training data to get the significance score $\s_l$ in the $l$-th layer;
		\STATE Set $\m_l=\m_{l+1}$, $\E_l=+\infty$, $r=0$; 
		\WHILE{$\E_{l+1} > \epsilon$ }		
		\STATE Set  $r$ elements in $\m_{l,i}$ to 1 according to positions of the largest $r$ scores $\s_{l,i}$. 
		\STATE Fine-tune $l$-th layer $\B_l(Z_{l-1})$ for a few epochs.
		\STATE Calculate error $\E_{l+1}$ in the $(l+1)$-th layer. 
		\STATE $ r=r+ r'$.
		\ENDWHILE
		\ENDFOR
		\ENSURE{The pruned vision transformer.}
	\end{algorithmic}

\end{algorithm}
\vspace{-3mm}

\subsection{Pruning Procedure}
Here we conclude the overall pipeline of the proposed patch slimming method. 

We start from the output layer and prune the previous layers layer-by-layer from top to down. Specially, all the patches preserved in the $(l+1)$-th layer will be also preserved in the $l$-th layer.  The other patches are greedily selected according to their impact scores $\s_{l,i}$, wherein patches with larger scores are preserved preferentially. Considering the reconstruction error $\E_{l+1}$ in the $(l+1)$-th layer is directly affected by the patch selection in the $l$-th layer, we use it to determine whether the $l$-th layer has already enough patches. In practice, we iteratively select $r'$ important patches in each step and continue the selection process in the current layer until $\E_{l+1}$ is less than the given tolerate value $\epsilon$. To make $\E_{l+1}$ well maintain the representation ability of current preserved patches, we fine-tune the current block $\HB_l$ for a few epochs after each step of patch selection. Taking the original feature $Z_{l-1}$ in the $(l-1)$-th layer as input, and the reconstruction error $\E_{l+1}$ as the objective, the parameters in the current block $\HB_l$ are optimized. Note that the block $\HB_l$ is a very small model with only one MSA and one MLP modules, the fine-tune process is very fast.  After pruning, the mask $\m_l$ is fixed, and weight parameters in the  vision transformer is further fine-tuned to be compatible with the efficient architecture.  The procedure of patching slimming for vision transformer is summarized in Algorithm~\ref{alg}.

\subsection{A Dynamic Variant}

In the above procedure,  whether a patch will be preserved is determined by the statistics over the training dataset.  It exploits the commonalities of redundant filters from different input adequately.  Besides, dynamic pruning  is the improved version of static pruning methods, which selects different patches for each input image. The dynamic strategy has been widely  explored for reducing channels of CNN models~\cite{gao2018dynamic,hua2018channel, tang2021manifold}. Similarly, the proposed patch slimming paradigm can be easily extended to the dynamic variant~( dubbed as DPS-ViT), and here we present a simple implementation.  Following \cite{gao2018dynamic}, we insert a small module $\G$ in each block to predict which patch is effective. The module $\G$ composes of a downsampling layer, linear layer and activation functions, which takes the input feature $Z_{l-1}\in\R^{N\times d}$ as input and outputs the approximate significance score  $\hat \s_l \in \R^{N}$. Recalling that the score of a patch $\s_l$~(Eq.~\ref{eq-imp}) actually depends on the input images, the module $\G$ is trained to fit $\s_l$ calculated for each input instance in the training phase. At inference, only patches with large score  $\hat \s_{l,i}$ are required to calculate. The dynamic strategy finds redundant patches of vision transformers depending on input data, which can  excavate patch redundancy more adequately.

\vspace{-1mm}
\section {Experiments}
\label{sec-exp}
\vspace{-1mm}
In  this section, we empirically investigate the effectiveness of the proposed patch slimming methods for efficient vision transformers~(PS-ViT). We evaluate our method on the benchmark ImageNet~(ILSVRC2012)~\cite{imagenet} dataset, which contains 1000-class natural images, including 1.2M training images and 5k validation images.  The proposed method is compared with SOTA pruning methods and we also conduct extensive ablation studies to better understand our method.
   
\subsection{Experiments on ImageNet}

We conduct experiments on the standard ViT models~\cite{dosovitskiy2020image} (DeiT~\cite{touvron2020training}), an improved variant network T2T-ViT~\cite{yuan2021tokens} and the state-of-the art LV-ViT~\cite{jiang2021all}. 

\textbf{Implementation details.} For a fair comparison, we follow the training and testing settings in the original papers~\cite{touvron2020training, yuan2021tokens, yuan2021tokens}, and the patch slimming is implemented based on the official pre-trained models. The global tolerant error is select from \{0.01, 0.02\} to get models with different acceleration rates, and the search granularity $r$ is set to 10. We fine-tune the current block for 3 epochs after each iteration of patch selection. After determining the proper patches in each layer, the pruned transformers are fine-tuned following the training strategy in \cite{touvron2020training}. All the experiments are conducted with PyTorch~\cite{paszke2017automatic} and MindSpore~\cite{mindspore} on NVIDIA V100 GPUs.

\textbf{Competing methods.} We compare our patch slimming with several representative model pruning methods including CNN channel pruning methods~\cite{tang2020scop} and BERT pruning methods~\cite{goyal2020power}. SCOP~\cite{tang2020scop} is a SOTA network pruning method for reducing the channels of CNNs, and we re-implement it to reduce the patches in vision transformers. 
PoWER~\cite{goyal2020power} accelerates BERT inference by progressively eliminating word-vector. HVT~\cite{pan2021scalable} directly designs efficient vision transformer architectures by  progressively reducing the spatial dimensions through pooling operations.  

\begin{table}
	\centering
	\small 
	\caption{Comparisons with SOTA transformer models on ImageNet.} 
	\vspace{-2mm}
	\label{tab-sota}
	\begin{tabular}{l | c | c}
		\toprule[1.5pt] 		
		Model  & FLOPs (G)    & Top-1 Acc. (\%) \\ \hline  
		DeiT-S~\cite{touvron2020training}     & 4.6  & 79.8 \\
		DeiT-B~\cite{touvron2020training}      & 17.5 &  81.8 \\ \hline
		PVT-Small~\cite{wang2021pyramid}      & 3.8  & 79.8 \\
		PVT-Medium~\cite{wang2021pyramid}     & 6.7 & 81.2 \\
		PVT-Large~\cite{wang2021pyramid}      & 9.8  & 81.7 \\ \hline
		T2T-ViT-14~\cite{yuan2021tokens} &5.2&81.5\\
		T2T-ViT-19~\cite{yuan2021tokens}  &8.9&81.9\\
		T2T-ViT-24~\cite{yuan2021tokens} &14.1&82.3\\ \hline
		TNT-S~\cite{han2021transformer}  & 5.2  & 81.5 \\
		TNT-B~\cite{han2021transformer} & 14.1  & 82.9 \\ \hline
		Swin-T~\cite{liu2021Swin}         & 4.5  & 81.3 \\
		Swin-S~\cite{liu2021Swin}          & 8.7  & 83.0 \\
		Swin-B~\cite{liu2021Swin}           & 15.4  &83.5 \\ \hline
		LV-ViT-S~\cite{jiang2021all}           & 6.6  & 83.3 \\
		LV-ViT-M~\cite{jiang2021all}           & 16.0 &84.1 \\ \hline
		PS-LV-ViT-S~(Ours)          &  4.7 & 82.4\\
		DPS-LV-ViT-S~(Ours)           & 4.5  & 82.9 \\	
		PS-LV-ViT-M~(Ours)          & 8.6 & 83.5\\	
		DPS-LV-ViT-M~(Ours)           & 8.3 & 83.7\\			
		\bottomrule[1.5pt]
	\end{tabular}
	\vspace{-4mm}
\end{table}

\textbf{Experimental results.} The experimental results are shown in Table~\ref{tab-img}, where `PS-' and `DPS-' denote the proposed patch pruning method and its dynamic variant, respectively. We evaluate on three versions of DeiT~\cite{touvron2020training} with different model sizes, \ie, DeiT-Ti, DeiT-S, and DeiT-B. Our method achieve obviously higher performance compared to the existing methods. The SCOP method~\cite{tang2020scop} designed for CNNs achieve poor performance when applied for reducing patches in a vision transformer, implying simply migrating the channel pruning methods cannot work well. PoWER~\cite{goyal2020power} has a larger accuracy drop than our method, indicating the model compression method for NLP models is not optimal for CV models.  Compared to the vision transformer structure pruning method VTP~\cite{zhu2021visual}, our method investigates a new prospective by pruning patches and achieve higher accuracy with similar FLOPs. 

As for T2T-ViT model, our method can reduce the FLOPs by 40.4\% and only have a small accuracy decrease (0.4\%), which is much better than the compared PoWER method. This indicates that the patch-level redundancy exists in various vision transformer models and our method can well excavate the redundancy. 

We further conduct experiments on a SOTA transformer model, LV-ViT~\cite{jiang2021all}, and show the results in Table~\ref{tab-sota}. The results show that our patch pruning method also work well on LV-ViT, \eg, the dynamic patch slimming reduces the FLOPs of LV-ViT-M from 16.0G to 8.3G, still achieving 83.7\% top-1 accuracy. Its performance is also superior to other SOTA models such as Swin transformer~\cite{liu2021Swin}.

\subsection{Ablation Study}
We conduct extensive ablation studies on ImageNet  to verify the effectiveness of each component in our method. The DeiT-S model on the ImageNet dataset is used as the base model.

\begin{table}[t] 
	\centering
	\small 
	\caption{Learned patch pruning \vs~uniform pruning.} 
	\vspace{-2mm}
	\begin{tabular}{c|c|c|c}
		\toprule[1.5pt]
		
		\multirow{2}{*}{Method}  & {Top-1} &Top-5  &  FLOPs  \\
		  & {Acc. (\%)} &Acc. (\%)  & (G) \\ 	
		\hline 
		Baseline& 79.8&95.0&  4.6 \\
		Uniform pruning & 77.2 & 93.8  &  2.6 \\ 
		Ours & 79.4 &  94.7 &  2.6  \\ 		
		\bottomrule[1.5pt]	
	\end{tabular}
	\label{tab-even}
	\vspace{-4mm}
\end{table}

\begin{figure*}[t]
	\centering
	\small
	\begin{minipage}[t]{0.3\textwidth}
		\centering
		\includegraphics[width=0.9\linewidth]{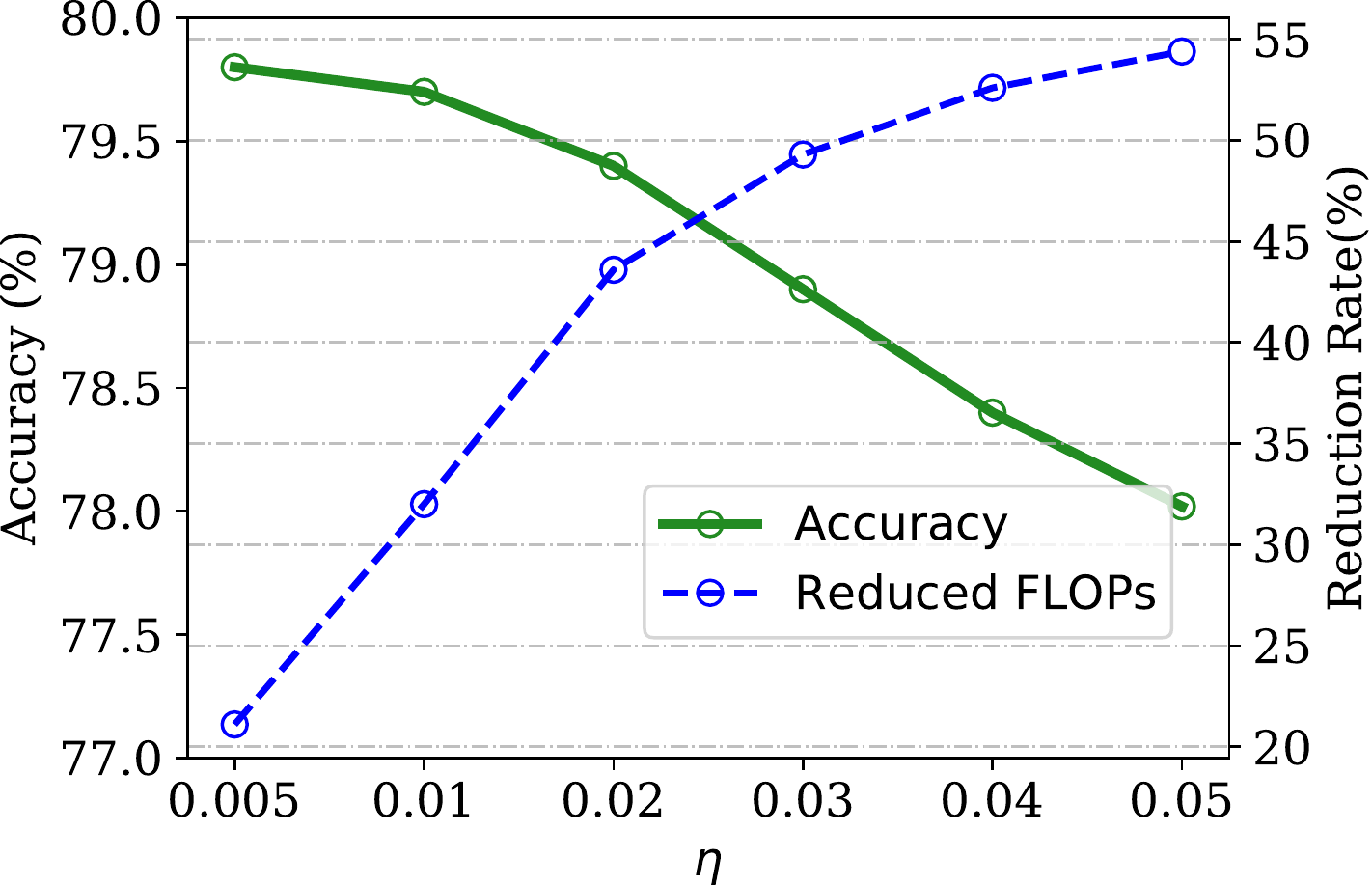}
		\caption{The ImageNet accuracy and FLOPs of the pruned DeiT-S \wrt tolerant error $\epsilon$.}
		\label{fig-eps} 
	\end{minipage}
	\hspace{0.5em}
	\begin{minipage}[t]{0.3\textwidth}
		\centering
		\includegraphics[width=0.9\linewidth]{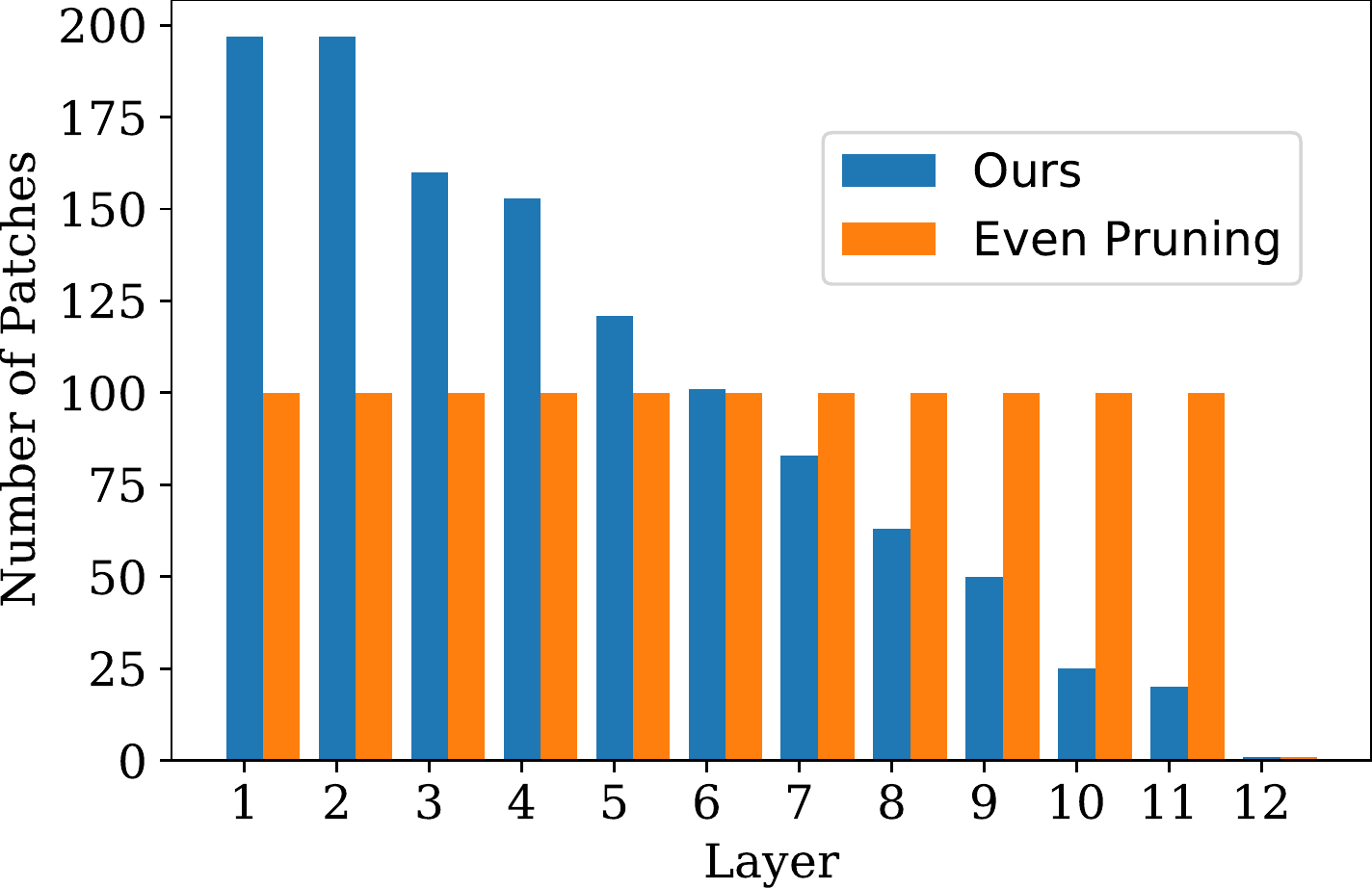}
		\caption{The architecture of pruned DeiT-S on ImageNet.}
		\label{fig-even} 
	\end{minipage}
	\hspace{0.5em}
	\begin{minipage}[t]{0.3\textwidth}
		\centering
		\includegraphics[width=0.9\columnwidth]{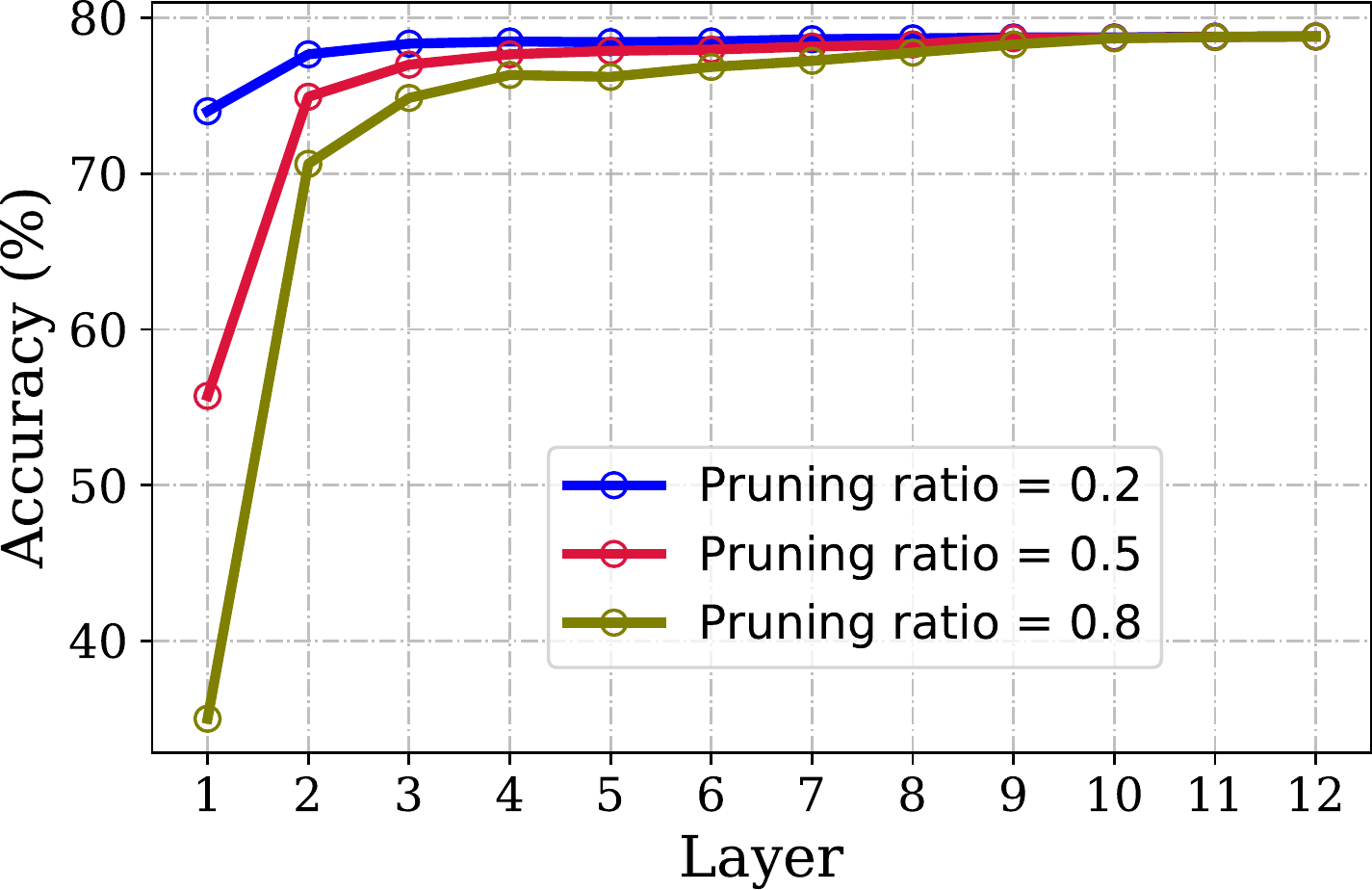} %
		\caption{Pruning different layers of the DeiT-S model on ImageNet.}
		\label{fig-layers} 
		\vspace{-5mm}
	\end{minipage}
	\vspace{-2mm}
\end{figure*}

\begin{figure*}[h]
	\centering
	\small
	\begin{subfigure}{0.3\linewidth}
		\centering
		\includegraphics[width=0.9\linewidth]{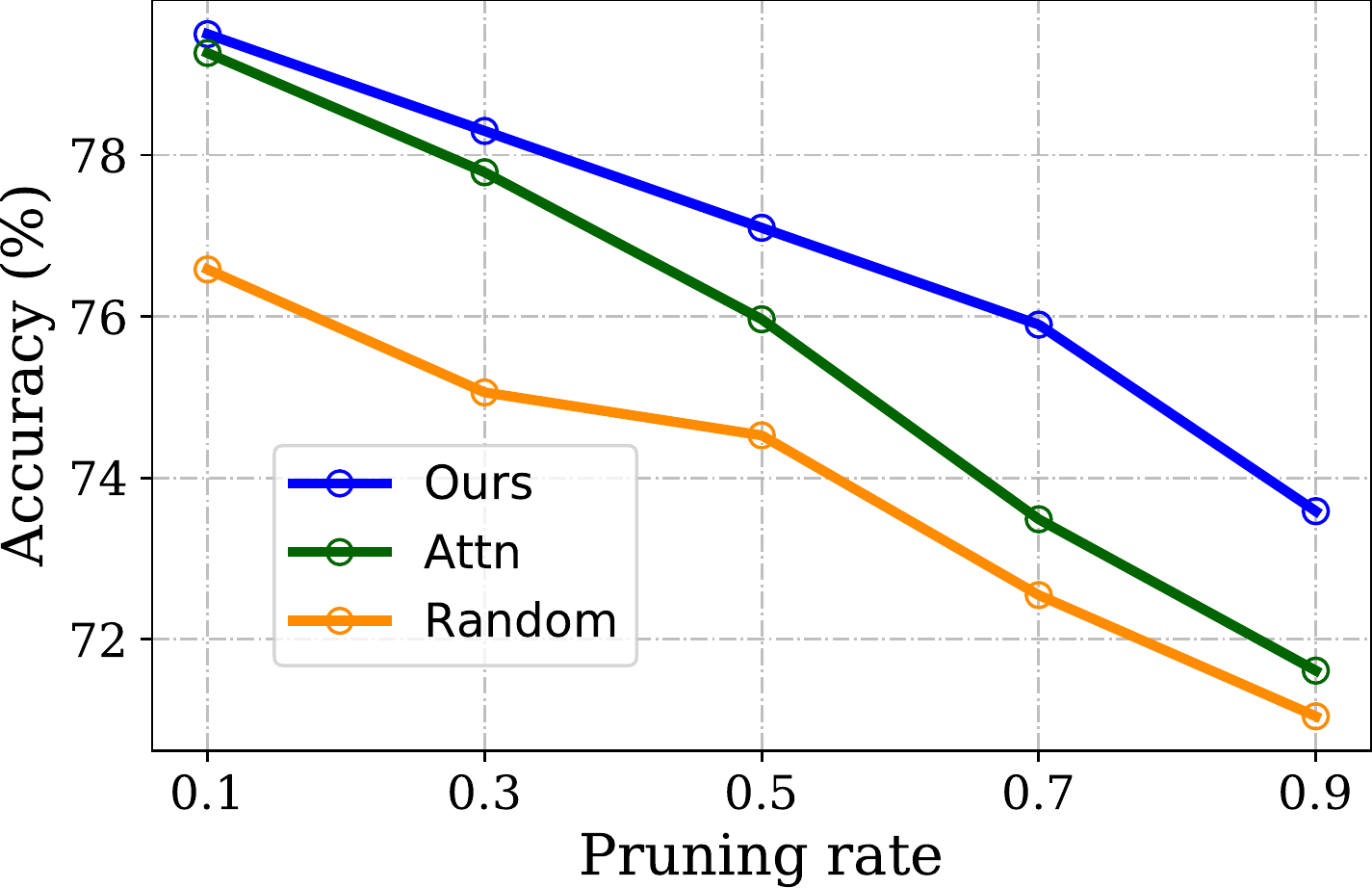}	%
		\caption{The 3-th layer.}
	\end{subfigure}
	\hspace{0.5em}
	\begin{subfigure}{0.3\linewidth}
		\centering
		\includegraphics[width=0.9\linewidth]{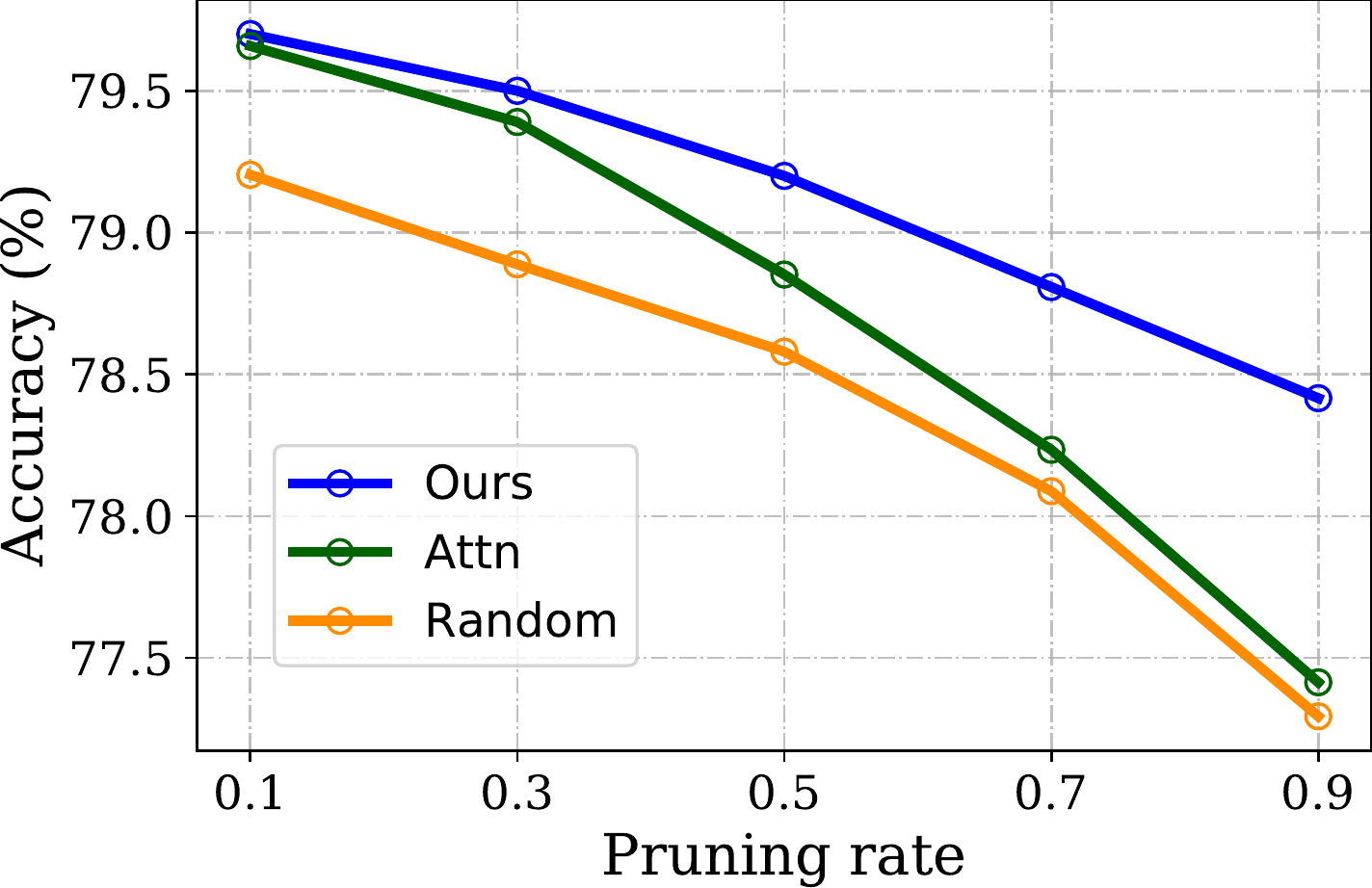}
		\caption{The 7-th layer.}
	\end{subfigure}
	\hspace{0.5em}
	\begin{subfigure}{0.3\linewidth}
		\centering
		\includegraphics[width=0.9\linewidth]{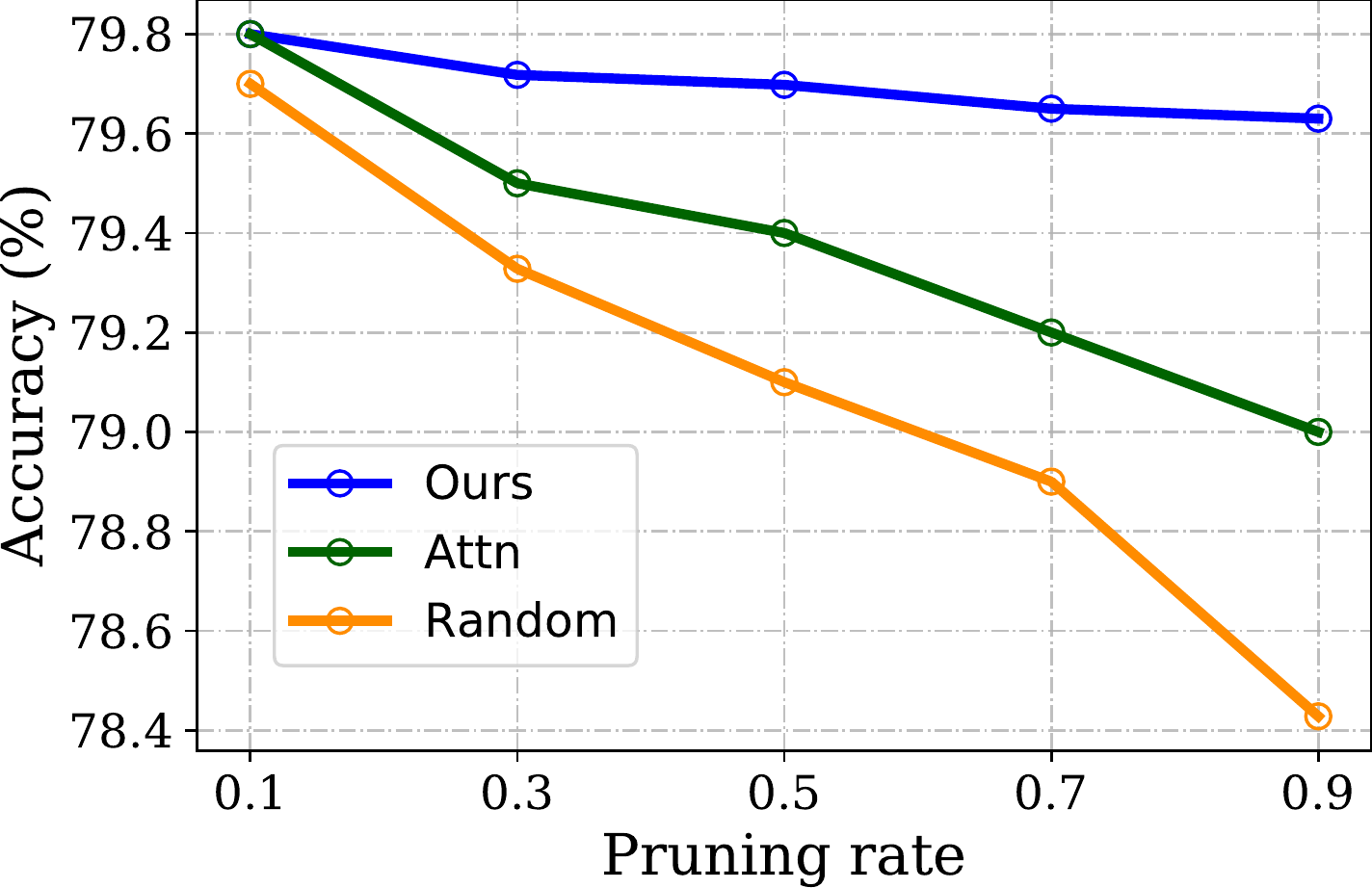}
		\caption{The 11-th layer.}
	\end{subfigure}
	\vspace{-2mm}
	\caption{Accuracy \wrt the pruning rate of  patches in a single layer.}
	\label{fig-acc}
	\vspace{-4mm}
\end{figure*}

\textbf{The effect of global tolerant error $\epsilon$.}
The tolerant error $\epsilon$ affects the balance between computational cost and accuracy of the pruned model, which is empirically investigated in Figure~\ref{fig-eps}.  Increasing $\epsilon$ implies larger reconstructed error between features of the pruned DeiT and original DeiT, while more patches can be pruned to achieve higher acceleration rate. When the reduction of FLOPs is less than 45\%, there is almost no accuracy loss~(less than 0.4\%), which is because that a large number of patches are redundant. 

\textbf{Learned patch pruning \vs~uniform pruning.}
In our method, the number of patches required in a specific layers is determined automatically via the global tolerant value $\epsilon$. The architecture of the pruned DeiT model is shown in  Figure~\ref{fig-even}. We can see that a pyramid-like architecture is obtained, where most of the patches in deep layers are pruned while more patches are preserved in shallow layers.  To validate the superiority of the learned pyramid architecture, we also implement a baseline that uniformly prunes all the layers with the similar pruning rate. We compare the results of the proposed patch slimming method and uniform pruning in Table~\ref{tab-even}. The accuracy of uniform pruning is only 77.2\%, which incurs a large accuracy drop (-2.6\%).  

To better understand the behavior of patch pruning in the vision transformer, we prune patches in a single layer to see how the test accuracy change. 
The experiments are conducted with DeiT-S model on ImageNet.

\textbf{Redundancy \wrt depth.} We test the patch redundancy of different layers to verify the motivation of top-down patch slimming procedure. We prune a single layer and keep the same pruning ratio for different layers. Figure~\ref{fig-layers} shows the accuracy of the pruned model after pruning patches of a certain layer, and each line denotes pruning patches with a given pruning rate. In deeper layers, more patches can be safely removed without large impact on the final performance. However, removing a patch in lower layers usually incurs obvious accuracy drop.  The patch redundancy is extremely different across layers and deeper layers have more redundancy, which can be attributed to that the attention mechanism aggregates features from different patches and the deeper patches have been fully communicated with each other. This phenomenon is different from the channel pruning in CNNs, where lower layers are observed to have more channel-level redundancy~(Figure~4 in \cite{he2017channel}).

\textbf{Effectiveness of impact estimation.} We define the scores $\s_l$ in Eq.~\ref{eq-imp} to approximate significance of a patch by propagating the reconstruction error of effective patches in output layer. To validate its effectiveness, we compare it with two baseline scores: `Random' denotes removing patches in the layer randomly, and `Attn' approximates the importance of a patch only with the norm of its attention map in the current layer. We compare the three scores by utilizing them to prune patches in different layer. The results are presented in Figure~\ref{fig-acc}, where  $y$-axis is the test accuracy of the pruned models~(without fine-tuning). From the results, our impact estimation manner suffers less accuracy loss than the others with the same pruning rate~(\eg, 50\%). It implies that our method can effectively identify patches that really make contributions to the final prediction.

\vspace{-1mm}
\section{Conclusion} 
\vspace{-1mm}
We propose to accelerate vision transformers by reducing the number of patches required to calculate. Considering that the attention mechanism aggregates different patches layer-by-layer, a top-down framework is developed to excavate the redundant patches. The importance of each patch is also  approximated according to its impact on the effective output features.
After pruning, a compact vision transformer with a pyramid-like architecture is obtained. Extensive experiments on benchmark datasets validate that the proposed method can  effectively reduce the computational cost. In the future, we plan to combine the patch slimming methods with more compression technologies~(\eg, weight pruning, model quantization) to explore extremely efficient vision transformers.  

\noindent\textbf{Acknowledgment.} This work is supported by National Natural Science Foundation of China under Grant No.61876007, Australian Research Council under Project DP210101859 and the University of Sydney SOAR Prize.

{\small
\bibliographystyle{ieee_fullname}
\bibliography{egbib}
}

\end{document}